\documentclass{article}

\usepackage{amsmath}
\usepackage{amsthm}
\usepackage{PRIMEarxiv}

\usepackage[utf8]{inputenc} 
\usepackage[T1]{fontenc}    
\usepackage{hyperref}       
\usepackage{url} 
\usepackage{algorithm,algpseudocode}
\usepackage{booktabs}       
\usepackage{amsfonts}       
\usepackage{nicefrac}       
\usepackage{microtype}      
\usepackage{lipsum}
\usepackage{fancyhdr}       
\usepackage{graphicx}       
\graphicspath{{media/}}     

\newtheorem{Proposition}{Proposition}[section]

\newtheorem{theorem}{Theorem}[section]

\newtheorem{Lemma}{Lemma}[section]

\newtheorem{Definition}{Definition}[section]

\newtheorem{corollary}{Corollary}[section]

\newtheorem*{namedthm}{\namedthmname}

\newcounter{namedthm}
\makeatletter

\title{A Novel Approach in Solving Stochastic Generalized Linear Regression via Nonconvex Programming}

\author{
  ~~Vu Duc Anh$^{1}$\thanks{~~Co-first authors, contributed equally} \space,\space  Tran Anh Tuan$^{2*}$,~~ Tran Ngoc Thang$^{2}$,~~ Nguyen Thi Ngoc Anh$^{2}$\thanks{~~Corresponding author}\\
\textsuperscript{1}Nanyang Technological University, Singapore \\
\textsuperscript{2}Faculty of Mathematics and Informatics, Hanoi University of Science and Technology, Hanoi, Vietnam \\
    \texttt{\small 
    ducanh001@e.ntu.edu.sg,
    tuan.ta222171m@sis.hust.edu.vn,}  
    \\
    \texttt{\small
    thang.tranngoc@hust.edu.vn,
    anh.nguyenthingoc@hust.edu.vn} \\}

\begin{document}
\maketitle

\begin{abstract}
Generalized linear regressions, such as logistic regressions or Poisson regressions, are long-studied regression analysis approaches, and their applications are widely employed in various classification problems. Our study considers a stochastic generalized linear regression model as a stochastic problem with chance constraints and tackles it using nonconvex programming techniques. Clustering techniques and quantile estimation are also used to estimate random data's mean and variance-covariance matrix. Metrics for measuring the performance of logistic regression are used to assess the model's efficacy, including the F1 score, precision score, and recall score. The results of the proposed algorithm were over 1 to 2 percent better than the ordinary logistic regression model on the same dataset with the above assessment criteria.
\end{abstract}


\section{Introduction}


In the realm of regression analysis, generalized linear models (GLMs) have carved out a significant niche, particularly in solving classification challenges. This paper delves into an advanced application of GLMs, focusing on the stochastic generalized linear regression model. Unlike traditional methods, this study approaches the model as a stochastic problem, integrated with chance constraints. Our methodology is rooted in the use of nonconvex programming techniques, a departure from conventional convex approaches.

We integrate clustering techniques and quantile estimation to further enhance the model's robustness. These methods are instrumental in accurately estimating the mean and variance-covariance matrix of random data, a crucial step in analyzing stochastic problems. Our study applies these techniques and innovatively combines them with the stochastic generalized linear regression model.

A critical aspect of our research is evaluating the model's performance. In this context, we employ various metrics specifically tailored for logistic regression. These include the F1 score, precision score, and recall score. These metrics comprehensively evaluate the model's effectiveness in real-world scenarios.

Logistic regression stands out as a crucial analytical method, as demonstrated by \cite{Bishop2007}, who illustrated the construction and updating of a logistic model for tackling classification issues. \cite{Jurafsky} applied logistic regression in the context of classifying multiple image classes, serving as a classification algorithm for predicting labels within a discrete set of classes. Its widespread applications include identifying email spam (as investigated by \cite{DADA2019e01802}), detecting internet fraud (explored by \cite{Sharifi}), and forecasting short-term electricity load  (explored by \cite{9242910}).

To enhance the effectiveness of logistic regression in handling large datasets with speed and precision, robust tools have been employed. Various convex programming algorithms play a role in finding optimal solutions, as discussed by \cite{LOBO1998193}, who also delved into the theoretical foundations of optimization problem-solving methods. Challenges arise in maintaining the consistency of problem constraints in the face of stochastic events affecting attribution data. To address this, \cite{Charnes1959}, \cite{Charnes1963}, and \cite{9778084}  introduced chance-constrained programming specifically designed to overcome such difficulties. In the pursuit of solving optimization problems related to logistic regression, the CVXOPT package by L. Vandenberghe for convex optimization problems has proven instrumental.

The structure of our research paper is methodically organized as follows: In the second section of our investigation, we provide some general preliminary findings. Section 3 contains the specifics of our suggested technique.
Within Section 4, we provide the findings, and within Section 5, we talk about the advantages of our approach. In the sixth section, we provide our concluding observations as well as instructions for further research.

\section{Preliminaries}


The general loss function of the logistic regression model which is proposed by \cite{Jurafsky}:\\ 

\begin{align}\label{ojt1} - \log\bigg(P(\mathbf{y}|\mathbf{X}, \mathbf{w})\bigg)= - &\sum_{i=1}^{\mathbf{N}}\bigg(y_{i}\log(\widehat{y_{i}})+ (1-y_{i})\log(1-\widehat{y_i})\bigg)
\end{align}
where
$\mathbf{X}=(x^{(1)},x^{(2)},\dots,x^{(\mathbf{N})}), x^{(i)} \in \mathbb{R}^\mathbf{d}$
$\mathbf{y}=(y_1,y_2,\dots,y_N), y_i \in \{0,1\}$, $\mathbf{w}=(w_1,w_2,\dots,w_d), \mathbf{w} \in \mathbb{R}^\mathbf{d}$ $\widehat{y}_i=\theta(\mathbf{w}^Tx^{(i)}), \theta $ \textit{ is activation function}
$\mathbf{d}$ is the number of variables, and
$\mathbf{N}$ is the number of observations.\\\\
\indent Loss function \eqref{ojt1} represent the \textit{cross-entropy loss} between $y_i$ and $\widehat{y_i}$. Furthermore, \cite{Jurafsky} showed also the loss function of the logistic sigmoid regression model which is modelized by \textit{sigmoid function}
\begin{align*}\tag{2}\label{obj2}
    J(\mathbf{w})=&-\sum_{i=1}^{\mathbf{N}}\bigg(y_{i}\log\left(\sigma\left(-\mathbf{w}^{T}x^{(i)}\right)\right)+
    (1-y_{i})\log\left(1-\sigma\left(-\mathbf{w}^{T}x^{(i)}\right)\right)\bigg), 
\end{align*}
where
$\mathbf{X}=(x^{(1)},x^{(2)},\dots,x^{(\mathbf{N})}), x^{(i)} \in \mathbb{R}^\mathbf{d}$
$\mathbf{y}=(y_1,y_2,\dots,y_N), y_i \in \{0,1\}$ $\mathbf{w}=(w_1,w_2,\dots,w_d), \mathbf{w} \in \mathbb{R}^\mathbf{d}$ $\sigma\left(-\mathbf{w}^{T}x^{(i)}\right)=\dfrac{1}{1+e^{(-\mathbf{w}^{T}x^{(i)})}}, $\textit{ is sigmoid function}
$\mathbf{d}$ is the number of variables
$\mathbf{N}$ is the number of observations.
\begin{figure}[h!]
    \centering
    \includegraphics[scale=0.6]{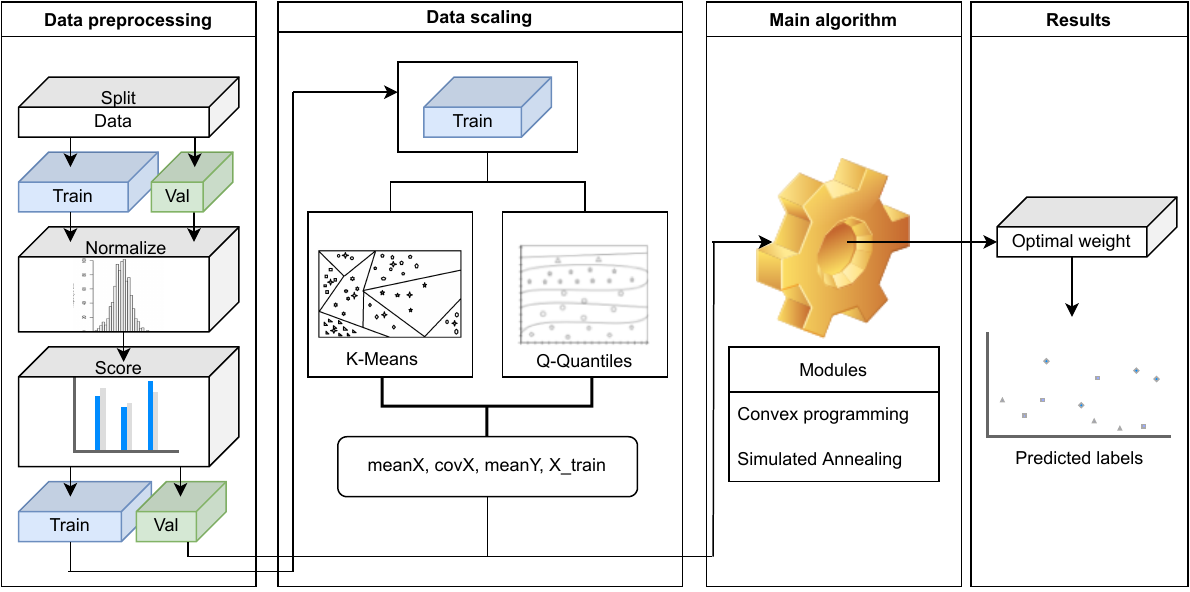}
    \caption{The general architecture of proposed model.}
    \label{model}
\end{figure}
\begin{Definition}\label{df1}
    (\cite{Fenchel1953ConvexCS}). Consider a convex set $\Omega\subseteq \mathbb{R^n}$, a continuously function $J(w)$ is said to be convex on $\Omega$ if, with $\forall w^{(1)},w^{(2)}\in \Omega \text{ and } \alpha \in \left[0,1\right]$, the following inequality holds true:
    \begin{align*}
        J\left(\alpha w^{(1)}+(1-\alpha)w^{(2)}\right)\le \alpha J(w^{(1)})+(1-\alpha)J(w^{(2)}),
    \end{align*}
then, $-J(w)$ is a concave function.
\end{Definition}
\begin{Definition}
    (\cite{avriel1988generalized}). A continuously function $J:\mathbb{R}^{n}\rightarrow\mathbb{R}$ is said to be semi-strictly quasi-convex on a convex set $\Omega$ if for every pair of distinct $w^{(1)},w^{(2)} \in \Omega, \lambda \in (0,1)$, we have
\begin{align*}
    \hspace{-0.5cm}J\left(\lambda w^{(1)}+(1-\lambda)w^{(2)}\right)\ge J(w^{(1)}) \Rightarrow J(w^{(2)})\ge J(w^{(1)}),
\end{align*}
then, $-J(w)$ is a semi-strictly quasi-concave function.
\end{Definition}
\begin{Definition}
    (\cite{avriel1988generalized}). With $\alpha\in\mathbb{R}$, a convex set $\Omega\subseteq\mathbb{R^n}$, let
    \begin{align*}
        \hspace{-1.cm}U^{(\alpha)}(J)=\{w\in\Omega|J(w)\ge\alpha\}\text{ is upper-level set of function J},\\
        \hspace{-1.cm}L^{(\alpha)}(J)=\{w\in\Omega|J(w)\le\alpha\}\text{ is lower-level set of function J}.
    \end{align*}
\end{Definition}
\begin{corollary}\label{c1}
(\cite{Fenchel1953ConvexCS}). Let $J$ be a convex function on $\Omega\subset\mathbb{R}^n$. The lower-level set of function $J$ is a convex set with $\forall\alpha\in\mathbb{R}$.
\end{corollary}
\begin{theorem}\label{tr1}
(\cite{DinhThe2005}). A continuous function $J$ is semi-strictly quasi-convex on a convex set $\Omega$ if and only if lower-level set $L^{(\alpha)}(J)$ is convex set with $\forall \alpha\in\mathbb{R}$.
\end{theorem}
\begin{theorem}\label{tr2}
(\cite{avriel1988generalized}). Let $\Phi_0$ be a positive convex function on $\Omega\subset\mathbb{R}^n$ and $\Phi_k$ be non-negative and concave functions on $\Omega$. Assume that $b\ge\sum_{i=1}^ka_i$ where $a_i>0.$ Then,
\begin{align*}
     F(x)=\Pi_{i=1}^k[\Phi_i(x)]^{a_i}/[\Phi_0(x)]^b,
\end{align*}
is semistrictly quasiconcave on $\Omega$.
\end{theorem}
\begin{Definition}\label{df2.4}
(\cite{Cottle1967LetterTT}). A quadratic form $\xi(w)=w^{T}Qw$ is convex function on convex set $\Omega\subseteq\mathbb{R}^n$, if only if $Q$ is  positive semi-definite, i.e.
\begin{align*}
    \hspace{-1.cm}\xi(w) \text{ is convex function on } \Omega \text{ if only if } \xi(w)\ge 0, \forall w\in\Omega. 
\end{align*}
\end{Definition}
\begin{theorem}\label{tr1}
        (\cite{Mangasarian1969}). A function $J(w)$ which is twice-differentiable is convex if and only if its hessian matrix (matrix of second-order partial derivatives) is positive semi-definite, i.e.
    \begin{align*}
        \text{with }\forall v : v^T\left(\nabla^T_{w}J(w)\right)v\ge 0,
    \end{align*}
\end{theorem}

\noindent where, $\nabla^T_{w}J(w)$ is the hessian matrix.
\begin{Proposition}\label{pr1}
    (\cite{Mangasarian1969}). Let $J(w)$ and $G(w)$ be two convex functions, $\lambda_1,\lambda_2$ are non-negative number. Then any linear combination of these two functions
    \begin{align*}
        \left(\lambda_1J+\lambda_2G\right)(w)=\lambda_1J(w)+\lambda_2G(w),
    \end{align*}
\end{Proposition}
\noindent is also a convex function.
\begin{Proposition}\label{p2.3}
(\cite{Boyd2006ConvexO}). Suppose $v:\mathbb{R}^k\to \mathbb{R}$, $l:\mathbb{R}^n\to\mathbb{R}^k$, consider $f=v\circ l:\mathbb{R}^n\to\mathbb{R}$. If both $v$ and $l$ are convex functions, $l$ is a non-decreasing function, then $f=v\circ l$ is convex function. 
\end{Proposition}
\begin{Lemma}\label{lm2}
    $J(\mathbf{w})$ in \eqref{obj2} is a convex function relatives to $\mathbf{w}$
\end{Lemma}
\begin{proof}
To prove $J(\mathbf{w})$ is convex function, we'll prove two component functions
\begin{align*}
    -\log\left(\sigma\left(-\mathbf{w}^{T}x^{(i)}\right)\right) \text{ and } -\log\left(1-\sigma\left(-\mathbf{w}^{T}x^{(i)}\right)\right),
\end{align*}
are convex functions relatives to $\mathbf{w}$.\\
Firstly, we have
\begin{align*}
    -\log\bigg(\sigma\left(-\mathbf{w}^{T}x^{(i)}\right)\bigg)&=-log\left(\dfrac{1}{1+e^{(-\mathbf{w}^{T}x^{(i)})}}\right)\\
    &=log\left(1+e^{(-\mathbf{w}^{T}x^{(i)})}\right),
\end{align*}
has the hessian matrix
\begin{align*}
    &\nabla^{2}_\mathbf{w} \bigg[-\log\bigg(\sigma\left(-\mathbf{w}^{T}x^{(i)}\right)\bigg)\bigg]\\
    &=\sigma\left(-\mathbf{w}^{T}x^{(i)}\right)\left(1-\sigma\left(-\mathbf{w}^{T}x^{(i)}\right)\right)x^{(i)}\left(x^{(i)}\right)^T.
\end{align*}
Now, we use \autoref{tr1}
\begin{align*}
    &\forall v: v^T\nabla^{2}_\mathbf{w} \bigg[-\log\bigg(\sigma\left(-\mathbf{w}^{T}x^{(i)}\right)\bigg)\bigg] v\\
    &=\bigg[\sigma\left(-\mathbf{w}^{T}x^{(i)}\right)\left(1-\sigma\left(-\mathbf{w}^{T}x^{(i)}\right)\right)x^{(i)}\left(x^{(i)}\right)^T\bigg]\\
    &=\sigma\left(-\mathbf{w}^{T}x^{(i)}\right)\left(1-\sigma\left(-\mathbf{w}^{T}x^{(i)}\right)\right)\bigg[\left(x^{(i)}\right)^Tv\bigg]^2 \ge 0,
\end{align*}
hence, $ -\log\bigg(\sigma\left(-\mathbf{w}^{T}x^{(i)}\right)\bigg)$ is convex function.\\
Secondly, we have also
\begin{align*}
        &-\log\bigg(1-\sigma\left(-\mathbf{w}^{T}x^{(i)}\right)\bigg)=-log\left(1-\dfrac{1}{1+e^{(-\mathbf{w}^{T}x^{(i)})}}\right)\\
    &=\mathbf{w}^{T}x^{(i)}+log\left(1+e^{(-\mathbf{w}^{T}x^{(i)})}\right),
\end{align*}
has the hessian matrix
\begin{align*}
    &\nabla^{2}_\mathbf{w} \bigg[-\log\bigg(1-\sigma\left(-\mathbf{w}^{T}x^{(i)}\right)\bigg)\bigg]\\
    &=\nabla^{2}_\mathbf{w} \bigg[-\log\bigg(\sigma\left(-\mathbf{w}^{T}x^{(i)}\right)\bigg)\bigg]\\
    &=\sigma\left(-\mathbf{w}^{T}x^{(i)}\right)\left(1-\sigma\left(-\mathbf{w}^{T}x^{(i)}\right)\right)x^{(i)}\left(x^{(i)}\right)^T,
\end{align*}
use above proof, we have $ -\log\bigg(1-\sigma\left(-\mathbf{w}^{T}x^{(i)}\right)\bigg)$ is also convex function.\\
Last, use \autoref{pr1} with $y_i \ge 0$ and $1-y_i \ge 0$, we obtain
\begin{align*}
\hspace{-1.cm}-\sum_{i=1}^{\mathbf{N}}\bigg(y_{i}\log\left(\sigma\left(-\mathbf{w}^{T}x^{(i)}\right)\right)+(1-y_{i})\log\left(1-\sigma\left(-\mathbf{w}^{T}x^{(i)}\right)\right)\bigg), 
\end{align*}
is convex function relatives to $\mathbf{w}$.
\end{proof}


\section{Proposed Methodology}
\subsection{Main Methodology}
The fundamental idea is to combine additional information from the mean and variance of data, then utilize probability formulas to transform a deterministic problem into a problem with chance constraints.\\
\noindent Let $\varepsilon_{i}=\mathbf{w}^{T}x^{(i)}$, modelize minimize loss function \eqref{obj2} into problem \eqref{p1}:
\begin{align*}\tag{$P_1$}\label{p1}
    &\hspace{-1.5cm}-\sum_{i=1}^{\mathbf{N}}\bigg(y_{i}\log\left(\sigma\left(-\varepsilon_i\right)\right)+
    (1-y_{i})\log\left(1-\sigma\left(-\varepsilon_i\right)\right)\bigg)\\
    &s.t: \begin{cases}
    P\bigg(|\mathbf{w}^{T}\overline{x}^{(i)}-\varepsilon_{i}|\le \alpha_{i}\bigg)\ge \beta_{i}\\
         E\bigg[\mathbf{w}^{T}\overline{x}^{(i)}-\varepsilon_{i}\bigg]=0
    \end{cases},
\end{align*}
where $\alpha_{i},\dots,\alpha_{N}$ are the non-negative real numbers and $\beta_{1},\beta_{2},\dots,\beta_{N}$ are the probability level of each constraint occurring .\\
Use probability formulas, we have:
\begin{align*}
    &P\bigg(|\mathbf{w}^{T}\overline{x}^{(i)}-\varepsilon_{i}|\le \alpha_{i}\bigg)=P\bigg(-\alpha_{i}\le \mathbf{w}^{T}\overline{x}^{(i)}-\varepsilon_{i}\le \alpha_{i}\bigg)\\
    &=P\bigg(\mathbf{w}^{T}\overline{x}^{(i)}-\varepsilon_{i}\le \alpha_{i}\bigg) - P\bigg(\mathbf{w}^{T}\overline{x}^{(i)}-\varepsilon_{i}\le -\alpha_{i}\bigg)\\
    &=2P\bigg(\mathbf{w}^{T}\overline{x}^{(i)}-\varepsilon_{i}\le \alpha_{i}\bigg)-1 \ge \beta_i,
\end{align*}
hence,
\begin{align*}
P\bigg(\mathbf{w}^{T}\overline{x}^{(i)}-\varepsilon_{i}\le \alpha_{i}\bigg)\ge \dfrac{\beta_{i}+1}{2} \\
P\bigg(\mathbf{w}^{T}\overline{x}^{(i)}-\varepsilon_{i}\le -\alpha_{i}\bigg)\le \dfrac{1-\beta_{i}}{2},
\end{align*}
with,
\begin{align*}
    &-\sum_{i=1}^{\mathbf{N}}\bigg(y_{i}\log\left(\sigma\left(-\varepsilon_i\right)\right)+
    (1-y_{i})\log\left(1-\sigma\left(-\varepsilon_i\right)\right)\bigg)\\
    &=-\sum_{i=1}^{\mathbf{N}}\bigg(y_{i}\varepsilon_{i}-\log{\left(e^{\varepsilon_{i}}+1\right)}\bigg).
\end{align*}
We rewrote problem \eqref{p1} into problem \eqref{p2}:
\begin{align*}\tag{$P_2$}\label{p2}
 \underset{[\mathbf{w},\varepsilon]}{\min}J(\mathbf{w},\varepsilon)=-\sum_{i=1}^{\mathbf{N}}\bigg(y_{i}\varepsilon_{i}-\log{\left(e^{\varepsilon_{i}}+1\right)}\bigg)\\\\
    s.t: \begin{cases}
    P\bigg(\mathbf{w}^{T}\overline{x}^{(i)}-\varepsilon_{i}\le \alpha_{i}\bigg)\ge \dfrac{\beta_{i}+1}{2} \\\\
P\bigg(\mathbf{w}^{T}\overline{x}^{(i)}-\varepsilon_{i}\le -\alpha_{i}\bigg)\le \dfrac{1-\beta_{i}}{2}\\
         E\big[\mathbf{w}^{T}\overline{x}^{(i)}-\varepsilon_{i}\big]=0.
    \end{cases}
\end{align*}
Next, we assumed $\overline{x}_{j}^{(i)}$ is a normal random variable with mean is $m_{\overline{x}_{j}^{(i)}}$, we denoted:
\begin{itemize}
    \item $\mathbf{m}_{\overline{x}^{(i)}} = \bigg(m_{\overline{x}_{1}^{(i)}},\dots,m_{\overline{x}_{d}^{(i)}}\bigg)$ is sample mean of $\overline{x}^{(i)}$.
    \item $V_{\overline{x}^{(i)}}$ is sample variance-covariance symmetric matrix of $\overline{x}^{(i)}$ with shape $d\times d$.
    \item $\phi$ is the cumulative distribution function of the standard Gaussian$(0,1)$.
\end{itemize}
Then, we transformed chance constraints of problem \eqref{p2} and utilize the monotonic of the function $\phi$, it follows that:
\begin{align*}
     &P\bigg(\mathbf{w}^{T}\overline{x}^{(i)}-\varepsilon_{i}\le \alpha_{i}\bigg)\\
     &=P\Bigg(\dfrac{\mathbf{w}^{T}\overline{x}^{(i)}-\mathbf{w}^{T}\mathbf{m}_{\overline{x}^{(i)}}}{\sqrt{\mathbf{w}^{T}V_{\overline{x}^{(i)}}\mathbf{w}}}\le \dfrac{\alpha_{i}+\varepsilon_{i}-\mathbf{w}^T\mathbf{m}_{\overline{x}^{(i)}}}{\sqrt{\mathbf{w}^{T}V_{\overline{x}^{(i)}}\mathbf{w}}}\Bigg)\\
     &=\phi\Bigg(\dfrac{\alpha_{i}+\varepsilon_{i}-\mathbf{w}^T\mathbf{m}_{\overline{x}^{(i)}}}{\sqrt{\mathbf{w}^{T}V_{\overline{x}^{(i)}}\mathbf{w}}}\Bigg) \ge \dfrac{\beta_i +1}{2}\\
     &\Leftrightarrow\dfrac{\alpha_{i}+\varepsilon_{i}-\mathbf{w}^T\mathbf{m}_{\overline{x}^{(i)}}}{\sqrt{\mathbf{w}^{T}V_{\overline{x}^{(i)}}\mathbf{w}}} \ge \phi^{-1}\Bigg(\dfrac{\beta_i +1}{2}\Bigg)\\
     &\Leftrightarrow\phi^{-1}\Bigg(\dfrac{\beta_i +1}{2}\Bigg)-\dfrac{\alpha_{i}+\varepsilon_{i}-\mathbf{w}^T\mathbf{m}_{\overline{x}^{(i)}}}{\sqrt{\mathbf{w}^{T}V_{\overline{x}^{(i)}}\mathbf{w}}} \le 0,
\end{align*}
similarly,
\begin{align*}
     &P\bigg(\mathbf{w}^{T}\overline{x}^{(i)}-\varepsilon_{i}\le -\alpha_{i}\bigg)\\
     &=P\Bigg(\dfrac{\mathbf{w}^{T}\overline{x}^{(i)}-\mathbf{w}^{T}\mathbf{m}_{\overline{x}^{(i)}}}{\sqrt{\mathbf{w}^{T}V_{\overline{x}^{(i)}}\mathbf{w}}}\le \dfrac{-\alpha_{i}+\varepsilon_{i}-\mathbf{w}^T\mathbf{m}_{\overline{x}^{(i)}}}{\sqrt{\mathbf{w}^{T}V_{\overline{x}^{(i)}}\mathbf{w}}}\Bigg)\\
     &=\phi\Bigg(\dfrac{-\alpha_{i}+\varepsilon_{i}-\mathbf{w}^T\mathbf{m}_{\overline{x}^{(i)}}}{\sqrt{\mathbf{w}^{T}V_{\overline{x}^{(i)}}\mathbf{w}}}\Bigg) \le \dfrac{1-\beta_i}{2}\\
     & \Leftrightarrow\dfrac{-\alpha_{i}+\varepsilon_{i}-\mathbf{w}^T\mathbf{m}_{\overline{x}^{(i)}}}{\sqrt{\mathbf{w}^{T}V_{\overline{x}^{(i)}}\mathbf{w}}}-\phi^{-1}\Bigg(\dfrac{1-\beta_i}{2}\Bigg)\le 0,
\end{align*}
and,
\begin{align*}
    E\bigg[\mathbf{w}^{T}\overline{x}^{(i)}-\varepsilon_{i}\bigg]=\mathbf{w}^T\mathbf{m}_{\overline{x}^{(i)}}-\varepsilon_i=0.
\end{align*}
Thus, problem \eqref{p2} equivalent with problem \eqref{p3}:
\begin{align*}\tag{$P_3$}\label{p3}
    \hspace{-1.cm}\underset{[\mathbf{w},\varepsilon]}{\min}J(\mathbf{w},\varepsilon)=-\sum_{i=1}^{\mathbf{N}}\bigg(y_{i}\varepsilon_{i}-\log{\left(e^{\varepsilon_{i}}+1\right)}\bigg)\\
    \hspace{-1.cm}s.t: \begin{cases}
         \phi^{-1}\Bigg(\dfrac{\beta_i +1}{2}\Bigg)-\dfrac{\alpha_{i}+\varepsilon_{i}-\mathbf{w}^T\mathbf{m}_{\overline{x}^{(i)}}}{\sqrt{\mathbf{w}^{T}V_{\overline{x}^{(i)}}\mathbf{w}}} \le 0\\\\
        \dfrac{-\alpha_{i}+\varepsilon_{i}-\mathbf{w}^T\mathbf{m}_{\overline{x}^{(i)}}}{\sqrt{\mathbf{w}^{T}V_{\overline{x}^{(i)}}\mathbf{w}}}-\phi^{-1}\Bigg(\dfrac{1-\beta_i}{2}\Bigg) \le 0 \\ \mathbf{w}^T\mathbf{m}_{\overline{x}^{(i)}}-\varepsilon_i= 0. 
    \end{cases}
\end{align*}
Or,
\begin{align*}\tag{$P_3$}
    \underset{[\mathbf{w},\varepsilon]}{\min}J(\mathbf{w},\varepsilon)=-\sum_{i=1}^{\mathbf{N}}\bigg(y_{i}\varepsilon_{i}-\log{\left(e^{\varepsilon_{i}}+1\right)}\bigg)\\
    s.t: \begin{cases}
         \phi^{-1}\Bigg(\dfrac{\beta_i +1}{2}\Bigg)-\dfrac{\alpha_{i}}{\sqrt{\mathbf{w}^{T}V_{\overline{x}^{(i)}}\mathbf{w}}} \le 0\\
        \dfrac{-\alpha_{i}}{\sqrt{\mathbf{w}^{T}V_{\overline{x}^{(i)}}\mathbf{w}}}-\phi^{-1}\Bigg(\dfrac{1-\beta_i}{2}\Bigg) \le 0. \\
        -\mathbf{w}^T\mathbf{m}_{\overline{x}^{(i)}}+\varepsilon_i\le 0\\
        \mathbf{w}^T\mathbf{m}_{\overline{x}^{(i)}}-\varepsilon_i\le 0.
    \end{cases}
\end{align*}
To make it convenient to follow, we let:
\begin{align*}
    \hspace{-1.cm}\mathbf{Z} & = [Z_1,Z_2,\dots,Z_{N+d}]^T=[w_0,\dots,w_d,\varepsilon_{d+1},\dots,\varepsilon_{N+d}]^{T}(Z\neq 0),
\end{align*}
so, problem \eqref{p3} was rewritten:
\begin{align*}\tag{$P_4$}\label{p4}
    &\hspace{-1.cm}\underset{\mathbf{Z}}{\min}J(\mathbf{Z})=-\sum_{i=1}^{\mathbf{N}}\Bigg(y_{i}Z_{d+i}-\log{\left(e^{Z_{d+i}}+1\right)}\Bigg)\\ 
    s.t: &\begin{cases}
     \phi^{-1}\left(\dfrac{1+\beta_{i}}{2}\right)-\dfrac{\alpha_{i}}{\sqrt{\mathbf{Z}^{T}V_{[\overline{x}^{(i)},\mathbf{K}^{(i)}]}\mathbf{Z}}}\le 0 \\\\
    \dfrac{-\alpha_{i}}{\sqrt{\mathbf{Z}^{T}V_{[\overline{x}^{(i)},\mathbf{K}^{(i)}]}\mathbf{Z}}}-\phi^{-1}\left(\dfrac{1-\beta_{i}}{2}\right)\le 0\\
    \mathbf{P_iZ}\le 0\\
    -\mathbf{P_iZ}\le 0,
    \end{cases} 
\end{align*}
where, 
\begin{itemize}
    \item $i=\overline{1,N}$.
    \item $\mathbf{P_{i}} = \bigg[\mathbf{m}_{\overline{x}^{(i)}},0_{d+1},\dots,-1_{d+i},\dots,0_{d+N}\bigg]$ is the $i^{th}$ coefficient vectors.
    \item $\phi$ is the cumulative distribution function of the standard Gaussian(0,1).
    \item $\beta_{1},\dots,\beta_{N}$ are the probability level of each constraint occurring.
    \item $\mathbf{m}_{\overline{x}^{(i)}}=\left(m_{\overline{x}_{1}^{(i)}},\dots,m_{\overline{x}_{d}^{(i)}}\right)$ is sample mean of $\overline{x}^{(i)}$.
    \item $V_{\bigg[\overline{x}^{(i)},\mathbf{K}^{(i)}_{(1\times N)}\bigg]}=\begin{bmatrix}
V_{\overline{x}^{(i)}} & \textbf{O}_{(N\times N)},\\
\textbf{O}_{(N\times N)} &\textbf{O}_{(N\times N)}
\end{bmatrix}$ is sample variance-covariance
symmetric matrix of $\overline{x}^{(i)}$ with shape $(N+d)\times (N+d)$.
    \item $\mathbf{K}^{(i)}$ is vector $(0_{d+1},\dots,1_{d+i},\dots,0_{N+d})$.
    \item \textbf{O} is zero matrix with shape $N\times N$.
\end{itemize}
\begin{Proposition}\label{p3.1}
(\cite{timm2002applied}). A matrix is a variance-covariance matrix if only if it is positive semi-definite.
\end{Proposition}
\begin{Lemma}\label{lm3.1}
    The constraints set of problem \eqref{p4} are convex set.
\end{Lemma}
\begin{proof}
    Consider $\Omega=\{Z\in\mathbb{R}^{N+d}:G_i^{(k)}(Z)\le 0,k=\overline{1,4}\}$ is feasible region (the constraints set) of problem \eqref{p4}.\\
    Let
    \begin{align*}
        &G_i^{(1)}(Z) =  \phi^{-1}\left(\dfrac{1+\beta_{i}}{2}\right)-\dfrac{\alpha_{i}}{\sqrt{\mathbf{Z}^{T}V_{[\overline{x}^{(i)},\mathbf{K}^{(i)}]}\mathbf{Z}}}\\
        &G_i^{(2)}(Z) = \dfrac{-\alpha_{i}}{\sqrt{\mathbf{Z}^{T}V_{[\overline{x}^{(i)},\mathbf{K}^{(i)}]}\mathbf{Z}}}-\phi^{-1}\left(\dfrac{1-\beta_{i}}{2}\right)\\
        &G_i^{(3)}(Z) = \mathbf{P_iZ}\\
        &G_i^{(4)}(Z) = -\mathbf{P_iZ}.
    \end{align*}
Now, we prove $G_i^{(1)}(Z)$ and $G_i^{(2)}(Z)$ are the semi-strictly quasi-convex functions on $\Omega$.\\\\
In fact, $\dfrac{\alpha_{i}}{\sqrt{\mathbf{Z}^{T}V_{[\overline{x}^{(i)},\mathbf{K}^{(i)}]}\mathbf{Z}}}=\dfrac{\sqrt[3]{\alpha_{i}^3}}{\sqrt{\mathbf{Z}^{T}V_{[\overline{x}^{(i)},\mathbf{K}^{(i)}]}\mathbf{Z}}}$ is a semi-strictly quasi-concave function on $\Omega$, cause of $\alpha_i^3$ is a non-negative constant function. In addition, from \autoref{p3.1} and $Z\neq 0$ then $V_{[\overline{x}^{(i)},\mathbf{K}^{(i)}]}$ is positive definite, then use \autoref{df2.4} we implied $\mathbf{Z}^{T}V_{[\overline{x}^{(i)},\mathbf{K}^{(i)}]}\mathbf{Z}$ is a positive convex function. Thus, use \autoref{tr2}, choose $a_1=\dfrac{1}{3},b=\dfrac{1}{2}$, and collaborate with $\phi^{-1}\left(\dfrac{1+\beta_{i}}{2}\right)$ is a constant function, we imply $G_i^{(1)}(Z)$ is a semi-strictly quasi-convex functions on $\Omega$. Similarity, we obtain $G_i^{(2)}(Z)$ is a semi-strictly quasi-convex functions on $\Omega$.\\
From \autoref{tr1}, we have
\begin{align*}
    \hspace{-1.8cm}L^{(0)}\bigg(G_i^{(1)}(Z)\bigg)=\Bigg\{\mathbf{Z}\in\Omega\Bigg|\phi^{-1}\left(\dfrac{1+\beta_{i}}{2}\right)-\dfrac{\alpha_{i}}{\sqrt{\mathbf{Z}^{T}V_{[\overline{x}^{(i)},\mathbf{K}^{(i)}]}\mathbf{Z}}}\le 0\Bigg\},
\end{align*}
and
\begin{align*}
    \hspace{-1.8cm}L^{(0)}\bigg(G_i^{(2)}(Z)\bigg)=\Bigg\{\mathbf{Z}\in\Omega\Bigg|\dfrac{-\alpha_{i}}{\sqrt{\mathbf{Z}^{T}V_{[\overline{x}^{(i)},\mathbf{K}^{(i)}]}\mathbf{Z}}}-\phi^{-1}\left(\dfrac{1-\beta_{i}}{2}\right)\le 0\Bigg\},
\end{align*}
are two convex sets.\\
Furthermore, $G_i^{(3)}(Z)$ and $G_i^{(4)}(Z)$ are two linear functions, thus $G_i^{(3)}(Z)$ and $G_i^{(4)}(Z)$ are two convex functions. Consequently, based on \autoref{c1}, then
\begin{align*}
    L^{(0)}\bigg(G_i^{(3)}(Z)\bigg)=\{Z\in\Omega|\mathbf{P_iZ}\le 0\},
\end{align*}
and
\begin{align*}
    L^{(0)}\bigg(G_i^{(4)}(Z)\bigg)=\{Z\in\Omega|-\mathbf{P_iZ}\le 0\},
\end{align*}
are two convex sets.\\
In short, the feasible region (the constraints set) $\Omega$ of problem \eqref{p4} is a convex set.
\end{proof}
\indent From \autoref{lm3.1} and \autoref{lm2}, we implied problem \eqref{p4} is a convex programming problem. Hence, we used convex programming tools such as the Lagrange multiplier method to tackle the solving of optimal solutions.
\subsection{The main algorithm}
\indent In this section, we describe the main algorithms for our proposed model. But before applying the main algorithms, we scored $y$ data by the original Logistic Regression model and utilized clustering algorithms to estimate the sample mean and variance-covariance matrix.
\subsubsection{The sample mean and variance-covariance matrix estimation}
K-means clustering and Q-quantiles clustering were used in the experimentation with real datasets. We divided $X$ into groups using K-means clustering and then calculated the mean and variance-covariance matrices for each group. We also estimated the sample mean of $y$, which was in the appropriate group, as part of this addition.

\textit{A. K-Means clustering algorithm}

In the first approach, we used the K-means clustering algorithm.
Clustering is a popular tool for finding groups or clusters which have the same feature in multivariate data and has found lots of applications in biology (see \cite{Eisen1998}), medicine (see \cite{Can1987}), psychology, and economics (\cite{Boyko2019}).

We had difficulty finding the number of data in each cluster because of the randomness of the cluster centers at initialization. An obligatory way for the clustering algorithm is to ask for input on the number of clusters in advance, which demands knowledge of the underlying datasets.  K-Means is a simple unsupervised learning algorithm that solves difficult clustering problems. Hence, we must provide deterministic k clusters for solving a problem (see \cite{Kodinariya2013}).

\indent The variance-covariance matrix will not be computed if a cluster contains just one data point. As a result, we chose a predetermined number of clusters to be employed in the clustering process, ensuring that the number of data points in each cluster is more than one. It is possible to estimate the median value and the variance-covariance matrices by grouping the data from $X$ into a cluster. The original logistic regression model will be used to translate $y$ into a value in the interval $(0,1)$. Then, the value of $y$ is commensurate with $X$ in a cluster together is calculated by mean.

\noindent\textit{B. Q-Quantiles algorithm}

\indent The definition of quantile was proposed in the paper of \cite{Takeuchi2006}. On the contrary, the K-Means clustering algorithm follows as Quantile estimation is divided into quantile levels for data of $y$. Then, at the same level, the mean and variance-covariance matrix of $X$ will be estimated. To estimate Large-scale Data, we need to define Q-Quantiles, which are values that divide a finite set into Q subsets of the same size. In our algorithm, we use a uniform probability distribution, so Q-Quantiles have the values $\{1/q, 2/q,\dots, q-1/q\}$.

\indent Moreover, instead of using a uniform probability distribution, we approached a new method based on the ideas of \cite{Chen2020} to divide quantile levels compatibly. Similar to the K-Means clustering algorithm, the Q-Quantiles algorithm will be implemented ineffectively if any quantile level contains the unit observation.

\subsubsection{The pseudocode and procedure of the main algorithm}
The pseudocode of the main algorithm was demonstrated by Algorithm~\ref{alg3}.\\ Algorithm~\ref{alg1} is our proposed algorithm for solving the solution using convex programming, its inputs include the sample mean, the sample variance-covariance matrix of $X$ data, the sample mean of $y$ data, and $\alpha,\beta$. Its output is an optimal solution to Problem~\ref{p4}.
\begin{algorithm}[H]
  \caption{The pseudocode of the Stochastic Logistic Regression model (SLR)}\label{alg1}
  \begin{algorithmic}[1]
    \Require \text{Data input:} $meanX, meanY, covX, \alpha, \beta$
    \Ensure \text{The weight:} $w$
    \State Define the objective function ($F$) of Problem~\ref{p4};
    \State Calculate constraint functions ($G$), and the coefficient of freedom ($h$) from $meanX, meanY, covX$;
    \State Use $CVXOPT$ package
    \begin{align*}
        w = CVXOPT(F,G,h,dim);
    \end{align*}
    \State \textbf{return} $w$.
  \end{algorithmic}
\end{algorithm}
Algorithm~\ref{alg2} is the pseudocode of the Simulated Annealing algorithm for solving the optimal parameters $\alpha,\beta$, its inputs consist of the sample mean, the sample variance-covariance matrix of $X$ data, the sample mean of $y$ data, and $X\_train,y\_train$. Its outputs return the optimal $\alpha,\beta$ parameters, and the corresponding optimal solution.
\begin{algorithm}[H]
  \caption{The pseudocode of the Simulated Annealing algorithm (SA)}\label{alg2}
  \begin{algorithmic}[1]
    \Require \text{Data input:} $meanX, meanY, covX, X\_train, y\_train$
    \Ensure $\alpha_0,\beta_0,w_{0}$
    \State Set initial 
        \begin{align*}
            &\alpha_0\ge 0, \beta_0\in(0,1),\\
            & w_0=SLR(meanX, covX, meanY, \alpha_0, \beta_0),\\
        &f(w_0)=-log\_loss(w_0, X\_train, y\_train);
        \end{align*}
    \While{The iteration isn't satisfied}
        \State Generate new initial
        \begin{align*}
          &\alpha_{new},\beta_{new},\\
          &w_{new}=SLR(meanX,covX,meanY,\alpha_{new},\beta_{new}),\\
          &f(w_{new})=-log\_loss(w_{new}, X\_train, y\_train);
        \end{align*}
        \State Calculate $\Delta f=f(w_{new})-f(w_0)$;
        \If {$\Delta f\le 0,\alpha_{new}\ge 0,\beta_{new}\ge 0$}
            \State $w_0=w_{new}, \alpha_0 = \alpha_{new}, \beta_0 = \beta_{new}$,
            \If{The stopping criteria is met}
                \State Terminate the computation process;
            \Else{}
                \State The iteration isn't satisfied;
            \EndIf
        \Else{}
            \State Accept the new solution based on Metropolis,
            \State The iteration isn't satisfied;
        \EndIf
    \EndWhile
    \State \textbf{return} $\alpha_0,\beta_0,w_0$.
  \end{algorithmic}
\end{algorithm}
\begin{algorithm}[H]
  \caption{The pseudocode of the main algorithm}\label{alg3}
  \begin{algorithmic}[1]
    \Require \text{Data input:} $X, y$
    \Ensure $w_{opt},\alpha_{opt},\beta_{opt}$
    \State Split, normalize data into train, and validation set.
    \State Scoring $y\_{train}$ data
    \begin{align*}
        &coef = LogisticRegression(X\_train, y\_train),\\
        &y\_score = \dfrac{1}{1+e^{-coef^{T}X\_train}};
    \end{align*}
    \If {Kmean} 
        \State Enter $n_c$ clusters;
        \State Cluster $X$ data 
        \begin{align*}
          [X,\text{scope}]=Kmeans\left(X\_train,n_c\right);
        \end{align*}
    \Else { Quantile}
        \State Enter $n_q$ levels;
        \State Divide $y$ data 
        \begin{align*}
          [y,\text{scope}]=Quantile\left(y\_score,n_q\right);
        \end{align*}
    \EndIf
    \State Estimate $covX, meanX, meanY$
        \begin{align*}
          &covX={\big(X.group(\text{by=['scope']})\big).\text{cov}()},\\
          &meanX={\big(X.group(\text{by=['scope']})\big).\text{mean}()},\\
          &meanY={\big(y.group(\text{by=['scope']})\big).\text{mean}()};
        \end{align*}
    \State Use Simulated Annealing algorithm
    \begin{align*}
        w_0 = SA(meanX, meanY, covX, X\_train, y\_train);
    \end{align*}
    \State \textbf{return} $w_{opt}=w_0,\alpha_{opt}=\alpha_0,\beta_{opt}=\beta_0$.
  \end{algorithmic}
\end{algorithm}
In the training process for the best optimal cluster, we set the maximum cluster to less than half the entire observations, then conduct training from cluster 1 to the max cluster. The best solution will be preserved with the corresponding $\alpha,\beta,$ and the cluster if the accuracy score on the validation set is greater than the original model is. The training process was illustrated in Figure~\ref{model1}.
\section{Results}
\subsection{Data}
We used the Heart Failure Clinical Records dataset, and the Rice Osmancik Cammeo dataset to perform our experiments. UCI Machine Learning Repository site was used to get all datasets.

 \begin{figure}[h!]
    \centering
    \includegraphics[scale=0.7]{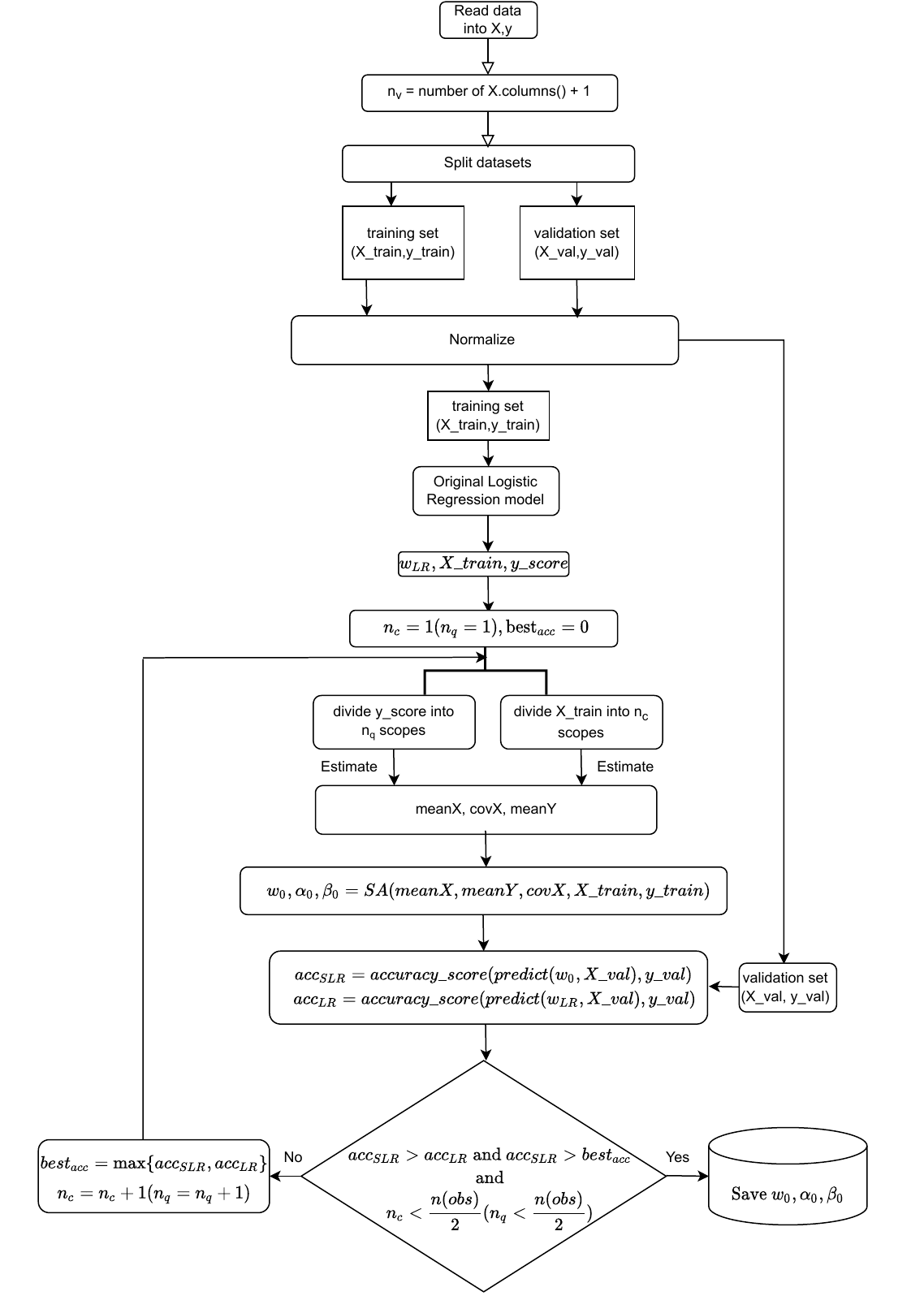}
    \caption{The training process.}
    \label{model1}
\end{figure}

\textbf{Heart Failure Clinical Records}

In April and December of 2015, the Faisalabad Institute of Cardiology and the Allied Hospital gathered medical records of 299 heart failure patients who were treated here. The patients, who ranged in age from 40 to 95 years old, were divided into two groups: 105 females and 194 males, respectively. They were all classified as classes III or IV of the New York Heart Association categorization of the stages of heart failure because they had left ventricular systolic dysfunction and prior heart failures. The dataset contains 13 attributes including Age, Anaemia, Creatinine phosphokinase, Diabetes, Ejection fraction, High blood pressure, Platelets, Serum creatinine, Serum sodium, Sex, Smoking, Time, and the prediction variable Death event.

\textbf{Rice Osmancik Cammeo}

For this research, the Osmancik and Cammeo rice varieties, which have been farmed in Turkey since 1997 and 2014, respectively, have been chosen. Osmancik species have broad, lengthy, glassy, and drab appearances when seen as a whole. The typical features of the Cammeo species are broad and long, glassy, and dull in appearance. Photos of rice grains were collected from the two species, which were then processed to get feature inferences from 3810 images. Each rice grain had its morphology analyzed for seven distinct morphological traits.
\subsection{Main results}
In the Heart Failure Clinical Records dataset, we used 70 percent of patients to train a stochastic logistic regression, and 30 percent of the remainder to verify it on the complete dataset. Furthermore, we compare our model's performance with results, which were mentioned by~\cite{Chicco2020} including metrics such as F1 score, Accuracy score, True positive rate, True negative rate, Precision-Recall curve, and Roc-AUC score. The formulas are shown in Table~\ref{t2}. Figure~\ref{fig:acc},~\ref{fig:f1},~\ref{fig:pre},~\ref{fig:re} perform the change of Accuracy score, F1 score, Precision score, and Recall score over the number of clusters (level quantiles) with the K-Means clustering algorithm, and Q-Quantiles algorithm correspondingly. In the 10th and 25th clusters, the accuracy score of our proposed model by both K-Means clustering and Q-Quantiles is outstanding compared to the original model. Besides, in the same above clusters, the F1 score, Precision score, and recall score are all greater than the metrics of the original model.

\begin{figure}[]
    \centering
    \begingroup
    \includegraphics[width=0.45\textwidth]{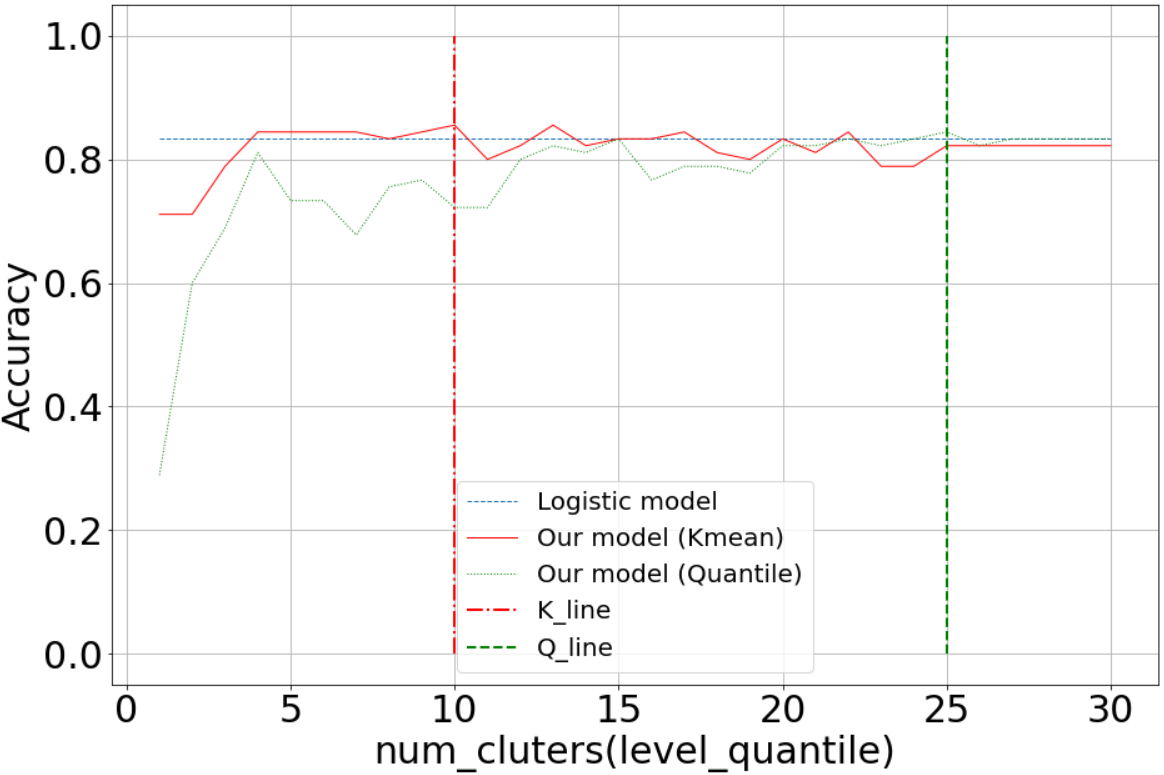}
    \caption{Accuracy score on Heart Failure Clinical Records dataset. }\label{fig:acc}
    \endgroup
\end{figure}
\begin{figure}[]
    \centering

    \begingroup
    \includegraphics[width=0.45\textwidth]{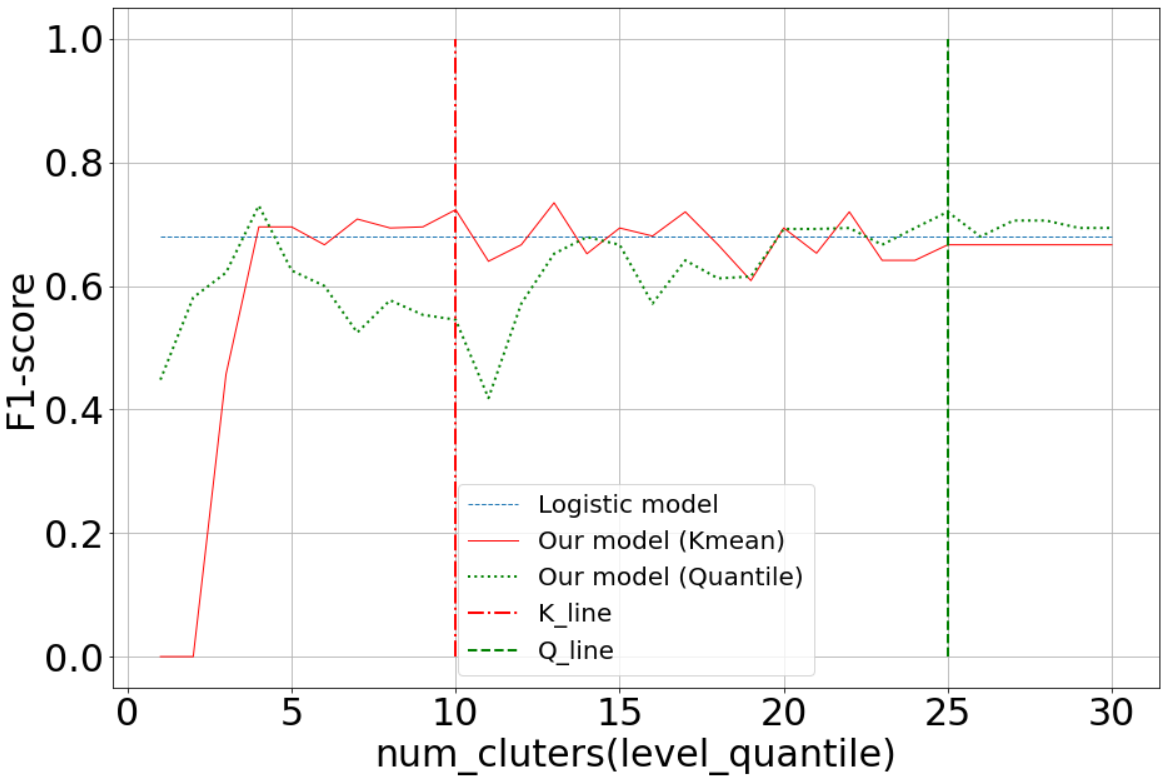}
    \caption{F1 score on Heart Failure Clinical Records dataset.}\label{fig:f1}
    \endgroup
\end{figure}
\begin{figure}[]
    \centering
    \begingroup
    \includegraphics[width=0.45\textwidth]{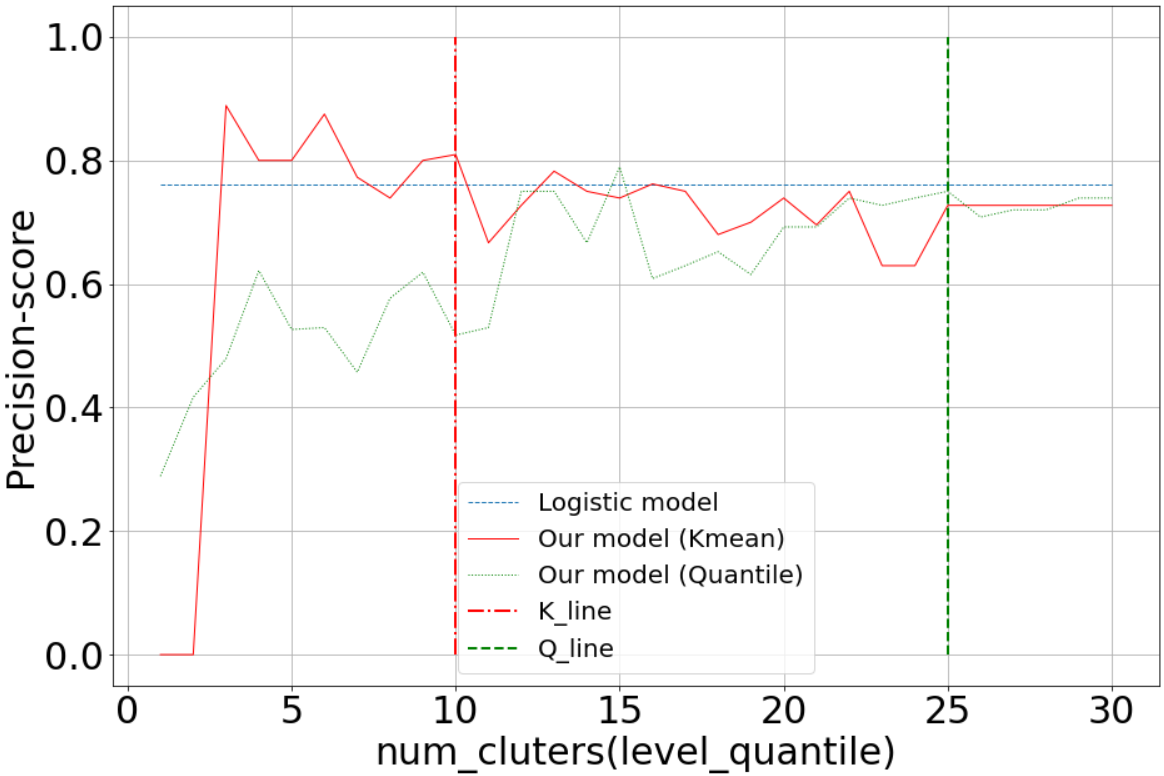}
    \caption{Precision score on Heart Failure Clinical Records dataset.}\label{fig:pre}
    \endgroup
\end{figure}
\begin{figure}[]
    \centering
    \begingroup
    \includegraphics[width=0.45\textwidth]
    {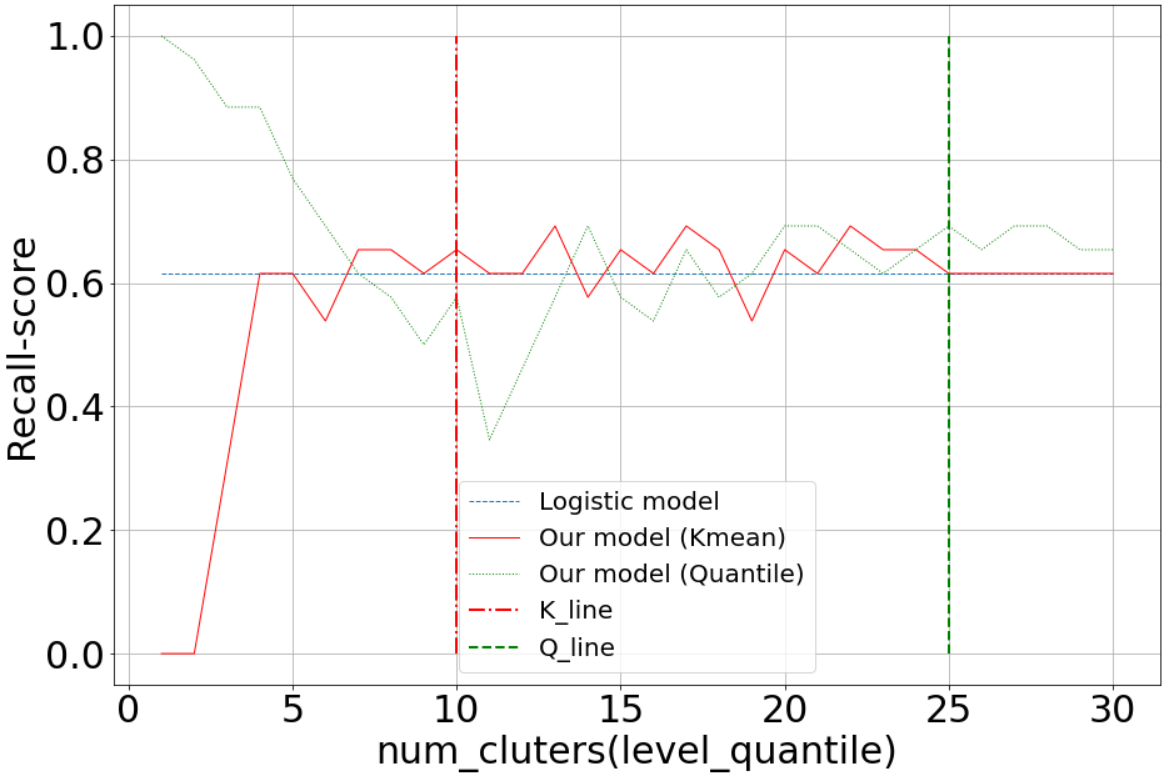}
    \caption{Recall score on Heart Failure Clinical Records dataset.}
    \endgroup
\end{figure}

\indent With the Rice Osmancik Cammeo dataset, we used 75 percent to train a stochastic logistic regression, and 25 percent to validate the dataset. In addition, we compare our model's performance with results, which were mentioned by~\cite{Cinar2019} including metrics in Table~\ref{t1}. The results are shown in Table~\ref{t3}, we trained the follow-up of 4 executions and calculated the mean of these evaluation metrics.

\indent Table \ref{t4} and Table \ref{t5} show the best of the evaluation metrics in $\alpha, \beta$, and the clusters (level quantiles) correspondingly, these metrics were greater and improved the ability prediction than the normal method was.
\begin{table*}[h]
\centering
\caption{Evaluation Metrics were based on TP: True Positive, TN: True negative, FP: False positive, and FN: False negative.}\label{t1}
\begin{tabular}{c|c|c|c}
\toprule
Metric & Formula & Metric & Formula \\
\midrule
Accuracy & $\dfrac{TP+TN}{TP+FP+TN+FN}$ & False Positive & $\dfrac{FP}{TN+FP}$ \\
 & &\\
Sensivitiy (TPR) & $\dfrac{TP}{TP+FN}$ & False Discovery & $\dfrac{FP}{TP+FP}$\\
 & &\\
Specificity (TNR) & $\dfrac{TN}{TN+FP}$ & False Negative & $\dfrac{FN}{TP+FN}$ \\
 & &\\
Precision & $\dfrac{TP}{TP+FP}$ & MCC & $\dfrac{TP\times TN-FP\times FN}{\sqrt{(TP+FP)(TP+FN)(TN+FP)(TN+FN)}}$\\
 & &\\
F1 score & $\dfrac{2TP}{2TP+FP+FN}$ & PR AUC & Precision-recall under the curve\\
 & &\\
Negative Predictive \\ Value & $\dfrac{TN}{TN+FN}$ & ROC AUC &  Receiver operating characteristic under the curve \\
\bottomrule
\end{tabular}

\end{table*}
\begin{center}
\begin{table*}[h!]
\caption{2 top rows: Mean follow-up time of 100 executions for survival prediction results by~\cite{Chicco2020}. 2 bottom rows: Top 2 models which have the best performance on all of the metrics follow-up time of 100 executions. The prediction results of our proposed model using the K-means clustering algorithm ($n_c = [1,30]$).}\label{t2}
\begin{tabular}{cccccccc}
\toprule

Model & MCC & F1 score & Accuracy & TPR & TNR & PR AUC & ROC AUC  \\
\midrule
Logistic regression & $+$0.616 & 0.719 & 0.838 & 0.785 & 0.860 &  0.617 &  0.822 \\
(3 selected features) & & & & & & &\\
Logistic regression & $+$0.607 & 0.714 & 0.833 & 0.780 & 0.856 & 0.612 & 0.818  \\
(all features) & & & & & & &\\
SLR \textbf{(ours)}& \bfseries\underline{$+$0.685 }& \bfseries0.790 & \bfseries\underline{0.850} & \bfseries0.785 & \bfseries\underline{0.891} & \bfseries\underline{0.709} & \bfseries0.834\\
(K-Means algorithm, all features) & & & & & & &\\
SLR \textbf{(ours)}& \bfseries$+$0.677 & \bfseries\underline{0.793} & \bfseries\underline{0.850} & \bfseries\underline{0.809} & \bfseries0.872 & \bfseries0.699 & \bfseries\underline{0.841}\\
(Q-Quantiles algorithm, all features) & & & & & & &\\
\bottomrule
\end{tabular}

\end{table*}

\begin{table*}[h!]
\caption{7 top rows: Mean follow-up time of 4 executions for survival prediction results by ~\cite{Cinar2019}. 2 bottom rows: Top 2 models which have the best performance on all of the metrics follow-up time of 4 executions. The prediction results of our proposed model using the K-means clustering (Q-Quantiles) algorithm ($n_c (n_q) = [1,30]$).}\label{t3}
\begin{tabular}{ccccccccccc}
\toprule

Model & Accuracy & Sensivitiy & Specificity & Precision & F1 score & NPV & FPR & FDR & FNR  \\
\midrule
LR & 0.930 & 0.923 & 0.936 & 0.914 & 0.918 &  0.942 &  0.064 & 0.087 & 0.077 \\
MLP & 0.929 & 0.922 & 0.934 & 0.910 & 0.916 & 0.942 & 0.066 & 0.090 & 0.078\\
SVM & 0.928 & 0.917 & \underline{0.937} & 0.915 & 0.916 & 0.938 & \underline{0.063} & 0.085 & 0.083  \\
DT & 0.925 & 0.912 & 0.935 & 0.913 & 0.912 & 0.934 & 0.065 & 0.087 & 0.088\\
RF & 0.924 & 0.914 & 0.932 & 0.908 & 0.911 & 0.936 & 0.069 & 0.092 & 0.086\\
NB & 0.917 & 0.909 & 0.923 & 0.896 & 0.902 & 0.933 & 0.077 & 0.104 & 0.091\\
k-NN & 0.886 & 0.864 & 0.903 & 0.871 & 0.867 & 0.897 & 0.097 & 0.129 & 0.136\\
SLR-K (\textbf{ours})& \bfseries\underline{0.934} & \bfseries0.954 & \bfseries0.908 & \bfseries\underline{0.918} & \bfseries\underline{0.944} & \bfseries0.934 & \bfseries0.093 & \bfseries\underline{0.065} & \bfseries0.046\\
SLR-Q (\textbf{ours})& \bfseries0.932 & \bfseries\underline{0.963} & \bfseries0.889 & \bfseries0.910 & \bfseries0.942 & \bfseries\underline{0.946} & \bfseries0.111 & \bfseries0.077 & \bfseries\underline{0.037}\\
\bottomrule
\end{tabular}

\end{table*}

\begin{table}[h!]%
\caption{Results of our proposed model on Heart Failure Clinical Records validation data.\label{t4}}
\centering
\begin{tabular}{cccccccccccc}
\toprule
Model & n-clusters & $\alpha$ & $\beta$ & Accuracy & F1 score & Precision score & Recall score  \\
& (level-quantiles) & & & & & \\
\midrule
SLR (\textbf{ours}) & 10 & 4.5 & 0.855 & \bfseries0.856 & \bfseries0.723 &  \bfseries0.810 &  \bfseries0.692 \\
(K-Means algorithm) & & & & & & \\
SLR (\textbf{ours})& 25 & 4.5 & 0.855 & 0.844 & 0.720 & 0.750 & 0.654  \\
(Q-Quantiles algorithm) & & & & & & \\
LR &-&-& - & 0.833 & 0.681 & 0.762 & 0.615  \\

\bottomrule
\end{tabular}
\end{table}

\begin{table}[h!]%
\caption{Results of our proposed model on Rice Osmancik Cammeo validation data.\label{t5}}
\centering
\begin{tabular}{cccccccccccc}
\toprule
Model & n-clusters & $\alpha$ & $\beta$ & Accuracy & F1 score & Precision score & Recall score  \\
 & (level-quantiles) & & & & & \\
\midrule
SLR (\textbf{ours}) & 12 & 0.3 & 0.62 & \bfseries0.938& \bfseries0.947 &  \bfseries0.944 &  \bfseries0.969 \\
(K-Means algorithm) & & & & & & \\
SLR (\textbf{ours}) & 31 & 4.5 & 0.855 & 0.937 & 0.946 & 0.935 & 0.960  \\
(Q-Quantiles algorithm) & & & & & & \\
LR &-&-& - & 0.933 & 0.942 & 0.937 & 0.948  \\

\bottomrule
\end{tabular}
\end{table}
\end{center}
\newpage

\section{Discussion}
Although numerous numerical approaches have been employed to solve the Original Logistic Regression issue, they were unable to eliminate the stochastic occurrences of data in practice. We may entirely explore a technique to change the chance constraints into the determined constraints by using the previously described chance constraints programming. The experiment's research results show that the effectiveness of the proposed model yields based on the evaluation metrics with the corporeal dataset.

We overcome the stochastic distribution of data that impacted the inability to occur with the acquired findings. Despite the extensive computations in the estimate procedure for the mean vectors or the covariance-variance matrix, the suggested approach outperforms contemporary machine learning algorithms.

To date, no techniques have addressed the solution of the logistic model by addressing the programming issue; instead, those studies have shown only a physical iterative formula that yields a proximal solution. With our novel concept, we were able to solve many of the different regression situations without knowing their solution formula.
\section{Conclusion and Future Work}

In this paper, we investigated the incorporation of stochastic elements into logistic regression through our proposed stochastic logistic regression model. To address the stochastic nature of the data, we employed the chance constraint introduced by Charnes and Cooper, treating each data point as a normal random variable. The key parameters governing the handling of the chance constraint are denoted as $\alpha$ and $\beta$, representing the acceptable range and probability. These parameters significantly influence the outcomes, as evident in the presented graphs. Despite some known asymptotic properties, selecting optimal values for these parameters remains a challenging task, especially considering the uniqueness of each dataset, leaving it as an open problem.


The importance of data scaling cannot be overstated when dealing with large-scale datasets, offering advantages such as computational speedup. Stochastic elements in the data are prevalent in real-world scenarios, either inherently or emerging post-scaling, as we treat groups of data points as new entities using methods like K-Means clustering or the Q-Quantiles algorithm discussed earlier.



The capability to navigate uncertainties and randomness provides a crucial advantage over deterministic methods. Through meticulous analysis and adjustments, the outcomes can be enhanced, making the approach applicable to a broader array of problems. We intend to enhance the proposed algorithm for solving the multiobjective variant of this problem in future endeavors. Numerous previous investigations have studied this multiobjective problem, such as optimizing a bicriteria convex programming problem's efficient set \cite{kim2013optimization} and implementing an outcome-based branch and bound algorithm \cite{thang2015outcome}. Furthermore, substantial contributions have been made to developing an outcome space algorithm for generalized multiplicative problems and optimization over the efficient set \cite{thang2016outcome}. Other related works include the research on optimizing over the Pareto set of semistrictly quasiconcave vector maximization and its application in stochastic portfolio selection \cite{vuong2023optimizing}, as well as the framework for controllable Pareto front learning with completed scalarization functions and its applications \cite{tuan2024framework}.



\bibliographystyle{unsrt}  
\bibliography{references} 

\end{document}